\documentclass[twoside,leqno,twocolumn]{article}
\usepackage{ltexpprt}

\usepackage{times}
\usepackage{helvet}
\usepackage{courier}
\usepackage{mathtools}
\usepackage{amsfonts}
\usepackage{bm}
\usepackage{microtype}
\usepackage{algorithm}
\usepackage{amsmath}
\usepackage{enumitem}
\usepackage{algpseudocode}
\usepackage{comment}
\usepackage[usenames,dvipsnames]{color}
\usepackage{graphicx}
\usepackage{subcaption}

\definecolor{red}{rgb}{0.8, 0.2, 0.2}
\definecolor{blue}{rgb}{0.2, 0.2, 1.0}
\definecolor{darkred}{rgb}{0.6, 0.25, 0.25}

\newcommand{\LS}[1]{{\color{blue}[LS: #1]}}

\newcommand{\mb}{\mathbf}

\begin{document}

\title{Efficient Online Relative Comparison Kernel Learning}
\author{Eric Heim \thanks{University of Pittsburgh, Department. of Computer Science \newline \{eric, milos\}$@$cs.pitt.edu}
\and
Matthew Berger \thanks{Air Force Research Laboratory, Information Directorate \newline \{matthew.berger.1, lee.seversky\}$@$us.af.mil}
\and
Lee M. Seversky \footnotemark[2]
\and
Milos Hauskrecht \footnotemark[1]
}
\date{}
\maketitle
%

\maketitle

\begin{abstract}
Learning a kernel matrix from relative comparison human feedback is an important problem with applications in collaborative filtering, object retrieval, and search.  For learning a kernel over a large number of objects, existing methods face significant scalability issues inhibiting the application of these methods to settings where a kernel is learned in an online and timely fashion. In this paper we propose a novel framework called \textbf{E}fficient online \textbf{R}elative comparison \textbf{K}ernel \textbf{LE}arning (ERKLE), for efficiently learning the similarity of a large set of objects in an online manner. We learn a kernel from relative comparisons via stochastic gradient descent, one query response at a time, by taking advantage of the sparse and low-rank properties of the gradient to efficiently restrict the kernel to lie in the space of positive semidefinite matrices.  In addition, we derive a passive-aggressive online update for minimally satisfying new relative comparisons as to not disrupt the influence of previously obtained comparisons. Experimentally, we demonstrate a considerable improvement in speed while obtaining improved or comparable accuracy compared to current methods in the online learning setting.

\end{abstract}
\section*{Keywords}
Online Learning, Kernel Learning, Relative Comparisons.

\section{Introduction}

Learning a similarity model over a set of objects from human feedback is important to many applications in collaborative filtering, document and multimedia retrieval, and visualization.  
It has been shown that by incorporating human feedback, the overall performance of such applications can be greatly improved \cite{koren2011ordrec,joachims2002optimizing,kovashka2012whittlesearch,levy2009music,zudilova2008trends}. In this work we focus on learning a similarity model from human feedback through relative comparisons.  More specifically, we focus on the \emph{relative comparison kernel learning} (RCKL) problem, in which the goal is to learn a positive semidefinite (PSD) kernel matrix from relative comparisons given by humans.  Kernels are used for modeling object relationships in many learning techniques~\cite{scholkopf2002learning}, and hence are applicable to many methods that utilize kernels for these applications.

In learning a kernel from human supervision, it is important to obtain feedback which is intuitive for the user to provide and informative for a learning algorithm to use. For instance, naive forms of supervision such as numerical judgments between pairs of objects have been shown to be very noisy~\cite{stewart2005absolute}. A \emph{relative comparison}, the response to a query of the form ``\textit{Is object A more similar to object B or C?}'', is well known as an intuitive mechanism for soliciting human feedback and an effective way of learning similarity~\cite{kendall1948rank}.
Recent works addressing fine-grained categorization~\cite{wah2014cvpr} and perceptual visualization design~\cite{demiralp2014infovis} have shown the practicality and benefit of learning kernels from relative comparisons.

Many RCKL methods \cite{agarwal2007generalized,van2012stochastic} learn a kernel by solving a semidefinite program (SDP) in batch, where all obtained relative comparisons are required to learn the kernel.  However, in many practical applications, a batch approach is not appropriate due to the online and dynamic nature of the application.  For example in crowdsourcing, it is often of interest to minimize the number of dispatched tasks and thus the cost of the crowd by leveraging active learning techniques~\cite{tamuz2011adaptively,jamieson2011low} to adaptively select the most informative relative comparison query. The success of these techniques depends on maintaining an up to date model so as to ensure the most informative query is selected, as well as an efficient learning method to quickly update the model so that no crowd participant is idle.  Likewise, recommendation systems for online marketplaces obtain continuous feedback in the form of click-through data via user interaction. In order for the learned kernel to be up to date and reflect the latest user feedback, the learning method must be able to quickly incorporate feedback as it is received.

These scenarios motivate the need for an \emph{efficient} and \emph{online} method for learning from large-scale relative comparison data. Batch methods poorly scale for large object collections primarily because they must ensure their solutions are PSD. Without any prior assumptions on the data this operation is of $O(n^3)$ time complexity for $n$ objects, which for large $n$ is prohibitively slow for the aforementioned applications.

This work introduces a novel online RCKL framework called \textbf{E}fficient online \textbf{R}elative comparison \textbf{K}ernel \textbf{LE}arning (ERKLE) that sequentially updates a kernel one query response at a time in $O(n^2)$ complexity. ERKLE employs stochastic gradient descent~\cite{bottou1998online} for RCKL, taking advantage of the sparse and low-rank structure of the RCKL gradient over a single comparison to devise fast updates that only require finding the smallest eigenvector and eigenvalue of a suitable matrix.
We show that the gradient structure, which enables such an efficient update, generalizes several well-known convex RCKL methods~\cite{agarwal2007generalized,van2012stochastic}.  The structure of the gradient also reveals a simple way to bound the smallest eigenvalue after each gradient step, which often allows updates to be performed in constant time. Motivated by work in online learning \cite{crammer2006online}, we also derive a passive-aggressive version of ERKLE to ensure learned kernels model the most recently obtained relative comparisons without over-fitting.  In summary, our main contributions are:



\begin{enumerate} [noitemsep]
     \item An online RCKL framework for large-scale similarity learning that generalizes many current RCKL methods.
     \item An efficient kernel update method with $O(n^2)$ time complexity that exploits the unique structure of RCKL stochastic gradients when stochastic gradient steps may result in a non-PSD matrix.
     \item A passive-aggressive update procedure for online relative comparison kernel learning
     \item An experimental evaluation that shows ERKLE has both improved performance and faster run times compared to batch RCKL methods.
\end{enumerate}



\section{Related Work}

The problem of learning a kernel matrix, driven by relative comparison feedback, has been the focus of much recent work. Most recent techniques primarily differ by the choice of loss function. For instance, Generalized Non-metric Multidimensional Scaling~\cite{agarwal2007generalized} employs hinge loss, Crowd Kernel Learning~\cite{tamuz2011adaptively} uses a scale-invariant loss, and Stochastic Triplet Embedding~\cite{van2012stochastic} uses a logistic loss function.

The aforementioned RCKL methods can be viewed as solving a kernelized special case of the classic non-metric multidimensional scaling problem \cite{kruskal1964nonmetric}, where the goal is to find an embedding of objects in $\mathbb{R}^d$ such that they satisfy given Euclidean distance constraints.  In contrast to many of the kernel-learning formulations, their analogous embedding-learning counterparts are non-convex optimization problems, which only guarantee convergence to a local minimum.  In the typical non-convex batch setting, multiple solutions are found with different initializations and the best is chosen among them.  This strategy is poorly suited for the online setting where triplets are being observed sequentially, and which solution is best may change as feedback is received.  



In this work we consider the \emph{online} RCKL problem, where one is sequentially acquiring relative comparisons among a large collection of objects.  Stochastic gradient descent techniques \cite{robbins1951stochastic} are a popular class of methods for online learning of high-dimensional data for a very general class of functions, where recent techniques~\cite{xiao2009dual,NIPS2012_4633} have demonstrated competitive performance with batch techniques.  In particular, recent methods~\cite{hazan2012projection,mahdavi2012stochastic} have developed efficient methods to solve SDPs in an online fashion.  The work of~\cite{chen2014efficient} shows how to devise efficient update schemes for solving SDPs when the gradient of the objective function is low-rank.
We build upon and improve the efficiency of this work, by taking advantage of the sparse and low-rank structure of the gradient common in convex RCKL formulations.  

Our passive-aggressive step size procedure is similar to that which is introduced in \cite{crammer2006online} for other online learning problems.  In their work, the authors create a passive-aggressive online update rule for classic SVM formulations used in problems such as binary/multi-class classification and regression.  In deriving such an update for different RCKL loss functions, we relate how different methods can be utilized under a common passive-aggressive framework.  To our knowledge, such an update for RCKL problems and the associated analysis of RCKL methods has not been done.

\section{Preliminaries}

In this section, we formally define RCKL and provide a brief overview of RCKL methods.  Let $S^n_+$ be the set of $n \times n$ PSD matrices, and $\mathbf{M}^{ab}$ be the entry at row $a$ column $b$ of a matrix $\mathbf{M}$.  The goal of RCKL is to learn a PSD kernel matrix $\mathbf{K} \in S^n_+$ over $n$ objects, given a set $\mathcal{T}$ of triplets:
 \begin{equation}
   \mathcal{T} = \{(a,b,c)\ |\ a \mathrm{\ is\ more\ similar\ to\ }b \mathrm{\ than\ }c\}
   \label{eq:TripletsDef}
 \end{equation}

\noindent such that squared distance constraints are satisfied:
\begin{equation}
   \begin{array}{rl}
       \forall_{\left(a,b,c\right) \in \mathcal{T}} : & d^2_{\mathbf{K}}(a,b) < d^2_{\mathbf{K}}(a,c)\\ [5 pt]
      \mathrm{where} & d^2_{\mathbf{K}}(a,b) = \mathbf{K}^{aa} + \mathbf{K}^{bb} - 2\mathbf{K}^{ab}.
   \end{array}
   \label{eq:KConstraints}
\end{equation}

\noindent We say a kernel $\mb{K}$ satisfies a triplet $t_i = \left(a_i,b_i,c_i\right) \in \mathcal{T}$ if the constraint in~\eqref{eq:KConstraints} corresponding to $t_i$ is satisfied.  

In this work, we consider triplets that are answers to relative comparison queries posed to one or more people.  We define a query $q$ to have three components, a ``head'' object $h$ to be compared with two objects $o^1$ and $o^2$.  A query $q = \left(h,\left\{o^1,o^2\right\}\right)$ can be answered by either the triplet $\left(h,o^1,o^2\right)$ or $\left(h,o^2,o^1\right)$, indicating that $h$ is more similar to $o^1$ than $o^2$ or $h$ is more similar to $o^2$ than $o^1$, respectively.   It is desirable to learn a kernel that not only satisfies observed triplets, but also that generalizes to unseen triplets, leading to a learned kernel that models a more complete notion of the desired human similarity space.


\subsection{RCKL Formulation} Many RCKL methods can be generalized by the following SDP:
 \begin{equation}
   \begin{array}{rl}
     \displaystyle \min_{\mathbf{K}} & \displaystyle L\left(\mathbf{K}, \mathcal{T}) + \tau\mathrm{Trace}(\mathbf{K}\right)\\ [5 pt]
     \mathrm{s.t.} & \mathbf{K} \succeq 0.
   \end{array}
   \label{eq:RCKL}
 \end{equation}

\noindent The objective function is composed of two terms.  The first term is a function $L$ measuring how much loss $\mathbf{K}$ incurs for not satisfying triplets in $\mathcal{T}$.  The second term is a trace regularization on $\mathbf{K}$ weighted by a hyperparameter $\tau$.  Trace regularization is used as a convex approximation of the non-convex rank function.  Higher values of $\tau$ enforce that \eqref{eq:RCKL} produces lower-complexity similarity models.  Finally, $\mathbf{K}$ is constrained to be PSD.

The loss function in the objective can be decomposed into the sum of losses over individual triplets:
\begin{equation}
  \displaystyle L\left(\mathbf{K}, \mathcal{T}\right) = \sum_{t \in \mathcal{T}}l\left(\mathbf{K},t\right).
  \label{eq:RCKLLoss}
\end{equation}
\noindent Existing RCKL methods differ in the choice of the loss function $l$.  The Stochastic Triplet Embedding (STE) approach of  \cite{van2012stochastic}  defines $l\left(\mathbf{K},t\right) = -\log p_{t}^{\mathbf{K}}$ as the loss function, where $p_{t}^{\mathbf{K}}$ is the probability that a triplet is satisfied:
\begin{equation}
   p_{t=\left(a,b,c\right)}^{\mathbf{K}} = \frac{\mathrm{exp}(-d^2_{\mathbf{K}}(a,b))}{\mathrm{exp}(-d^2_{\mathbf{K}}(a,b)) + \mathrm{exp}(-d^2_{\mathbf{K}}(a,c))}.
   \label{eq:STE}
\end{equation}
\noindent Generalized Nonmetric Multidimensional Scaling (GNMDS) \cite{agarwal2007generalized} uses a hinge loss, where $l\left(\mathbf{K},t=\left(a,b,c\right)\right)$ is defined as:
\begin{equation}
     \displaystyle \max(0, d^2_{\mathbf{K}}(a,b) - d^2_{\mathbf{K}}(a,c) + 1).
\label{eq:GNMDS}
\end{equation}
%

For either loss function $l$, \eqref{eq:RCKL} is a convex optimization problem and the globally optimal solution is found by performing projected gradient descent, which consists of two update steps.
The first step is a simple descending step along the gradient of the objective:
\begin{equation}
\mathbf{K}_i' = \mathbf{K}_{i-1} - \delta_i\left(\nabla L\left(\mathbf{K}_{i-1}, \mathcal{T}\right) + \tau \mathbf{I}\right),
\label{eq:batchStep}
\end{equation}
where $i$ denotes the current iteration, $\delta_i$ is the learning rate.
The second step  projects the result of the first gradient step onto the PSD cone:
\begin{equation}
\mathbf{K}_i = \Pi_{S_+}\left(\mathbf{K}_i'\right).
\end{equation}
These steps are iterated until convergence.



\section{Efficient Online Relative Comparison Kernel Learning (ERKLE)}

The main computational bottleneck of traditional RCKL methods is the projection onto the PSD cone, $\Pi_{S_+}$.  This projection is commonly found by first taking the eigendecomposition of $\mb{K}_i' = \mb{V} \mb{\Lambda} \mb{V}^T$  and setting all negative eigenvalues to $0$, i.e. $\mb{K}_i = \mb{V} [\mb{\Lambda}]_+ \mb{V}^t$, where $[\cdot]_+$ is defined entry-wise as $[\mb{\Lambda}^{ii}]_+ = \max(0,\mb{\Lambda}^{ii})$.  Absent of any prior knowledge on the structure of $\mb{K}_i'$, its full eigendecomposition is necessary for the projection.  Since this is an $O(n^3)$ operation, the projection step renders batch methods computationally prohibitive for learning the similarity of a large number of objects in an online manner.


\subsection{Stochastic Gradient Step}
To create an efficient and online framework for RCKL -- ERKLE -- we leverage stochastic gradient descent techniques \cite{bottou1998online}.  As shown in \eqref{eq:RCKLLoss}, the loss function $L$ naturally decomposes into the sum over losses $l$ defined on individual observations (triplets in our case).  From this decomposition, ERKLE first performs the following stochastic gradient step:
\begin{equation}
  \mathbf{K}_j' \gets \mathbf{K}_{j-1} - \delta_j\nabla l\left(\mathbf{K}_{j-1},t_j\right),
  \label{eq:stoUp}
\end{equation}
\noindent where triplets $t_1,...,t_{j-1}$ have been observed, $\mathbf{K}_{j-1}$ is the online solution after observing the $j-1$ triplet,

Performing a stochastic optimization gives ERKLE an advantage over current RCKL methods that perform batch optimizations.  Batch methods attempt to minimize a loss function over a training set.  This is known to minimize empirical risk with respect to the particular training samples, which is used as an estimate of expected risk over the ground truth distribution over all samples.  Obtaining triplets in an online fashion from a source can be viewed as sampling triplets from a ground truth distribution at random. As such, taking stochastic steps over samples directly minimizes expected risk with respect to the ground truth distribution of triplets, not empirical risk with respect to the training instances.  The practical impact of this characteristic is that stochastic methods tend to generalize better to unobserved samples. For more discussion on this characteristic of stochastic methods see \cite{bottou1998online}. 

Note that our online formulation does not include trace regularization.
Although this may impact our method in generalizing to unseen triplets, our online formulation achieves good generalization through carefully constructed, data-dependent step sizes $\delta_j$, as detailed in Section~\ref{subsec-pa}.

\subsection{Efficient Projection}
In order to retain positive semi-definiteness, after taking a stochastic gradient step the resulting matrix $\mb{K}_j'$ must be projected onto the PSD cone.
Following the procedure of $\Pi_{S_+}$ is prohibitively expensive for our online setting. Instead, for RCKL methods we can take advantage of the sparse and low-rank nature of the gradient to devise an efficient projection scheme.
To this end, we introduce a \emph{canonical gradient matrix} $\mb{G}$ over a triplet $t = \left(a,b,c)\right)$, where the entries are defined as:
\begin{equation}
\label{eq:G}
\mathbf{G}^{i j} = \begin{cases}
   -2 & \text{if } i = a, j = b \text{ or } i = b, j = a\\
   2  & \text{if } i = a, j = c \text{ or } i = c, j = a\\
   1  & \text{if } i = b, j = b\\
   -1 &  \text{if } i = c, j = c\\
   0 & \text{otherwise}.
  \end{cases}
\end{equation}
\noindent Now consider the following choice for the stochastic step:
\begin{equation}
    \nabla l\left(\mathbf{K},t \right) =  f\left(\mathbf{K},t\right) \mathbf{G},
    \label{eq:genUp}
\end{equation}

\noindent where $f$ is a real-valued function defined below.  With \eqref{eq:genUp} as the gradient in \eqref{eq:stoUp}, $\mb{K}_{j-1}$ is updated by increasing entries corresponding to the similarity between objects $a$ and $b$ and decreasing the similarity between $a$ and $c$ by a factor of $f(\mathbf{K}_{j-1},t_j)$. 

The function $f$ can be defined such that we recover the gradients of $l$ for different convex RCKL formulations.  The stochastic gradient for STE can be obtained by defining $f$ as:
\begin{equation}
  f\left(\mathbf{K},t\right) = 1 - p^{\mathbf{K}}_t
\end{equation}
\noindent Similarly by defining $f$ to be:
\begin{equation}
    f\left(\mathbf{K},t\right) = \left\{
  \begin{array}{rl}
    1 & \text{if }  d^2_{\mathbf{K}}(a,b) + 1 < d^2_{\mathbf{K}}(a,c)\\
    0 & \text{otherwise} 
    \label{eq:GNMDS}
  \end{array}
\right.
\end{equation}
\noindent  the stochastic gradient for GNMDS is obtained.  Note, that this not only generalizes these two methods for use in our online framework but also suggests a simple way to create new online RCKL methods by designing a function $f$ that weighs the contribution of individual triplets.


Decomposing the online updates in such a way reveals a key insight into how to perform efficient projections onto the PSD cone after the stochastic step.  Algorithm \ref{alg:ERKLE-Proj} outlines the procedure for efficient projection in ERKLE.  Here, $\lambda_{\downarrow}$ and $\mathbf{v}_{\downarrow}$ are the smallest eigenvalue and eigenvector of matrix $\mb{K}$, respectively.  This procedure has a time complexity $O(n^2)$ due to finding $\lambda_{\downarrow}$ and $\mathbf{v}_{\downarrow}$.  To show that Algorithm \ref{alg:ERKLE-Proj} does indeed perform a projection onto the PSD cone, we prove the following theorem:
\begin{theorem}
\label{thm:proj}
Algorithm \ref{alg:ERKLE-Proj} results in a PSD matrix $\mathbf{K}_j$ that is closest to $\mathbf{K}_{j}'$ in terms of Frobenius distance.
\end{theorem}
\begin{proof}
Let $\mathbf{K}_0 \in S^n_+$ (i.e. identity).  We use this as our base case and show inductively that after each iteration of the main loop, $\mathbf{K}_j$ remains PSD. Let $\gamma_j = \delta_jf\left(\mathbf{K}_{j-1},t_j\right)$ be the \emph{magnitude} of an update.  By \eqref{eq:genUp}, the update in Equation \eqref{eq:stoUp} can be written as $\mathbf{K}_{j-1} - \gamma_j\mathbf{G}$. The only nonzero eigenvalues of $-\gamma_j\mathbf{G}$ are $\lambda_1 = 3\gamma_j$ and $\lambda_2 = -3\gamma_j$.  It follows from Weyl's inequality that the matrix $\mb{K}_j' = \mb{K}_{j-1} - \gamma_j\mb{G}$ has \emph{at most} one negative eigenvalue.  If $\mathbf{K}_j'$ has no negative eigenvalues, then it is PSD (line 6 of Algorithm \eqref{alg:ERKLE-Proj}).  If $\mathbf{K}_j'$ has one negative eigenvalue, line 4 of Algorithm \ref{alg:ERKLE-Proj} results in a PSD matrix $\mathbf{K}_j$ that is closest to $\mathbf{K}_j'$ in terms of Frobenius distance by Case 2 of Theorem 4 in \cite{chen2014efficient}.
\end{proof}

The important implication of Thm. \ref{thm:proj} is that ERKLE can incorporate a triplet into a kernel in $O(n^2)$ time by performing the efficient projection outlined in Algorithm 
\ref{alg:ERKLE-Proj}.  Furthermore, if a step is sufficiently small, then no projection is needed at all.  Let $\lambda_j^0$ be the smallest eigenvalue of $\mathbf{K}_j$.  By Weyl's inequality, if 
$\lambda_j^0 - 3\gamma_j \geq 0$, then all eigenvalues of $\mathbf{K}_{j+1}'$ are greater than or equal to 0.  This can be used to skip the projection step when the update is known to result in a PSD 
matrix.
 In our 
algorithm, we lower bound the smallest eigenvalue by maintaining a conservative estimate $\hat{\lambda}_j^0$.  Initially, $\hat{\lambda}_0^0 \gets \lambda_0^0$. It is updated each iteration with it's 
lower bound $\hat{\lambda}_j^0 \gets \hat{\lambda}_{j-1}^0 - 3\gamma_j$.  If $\hat{\lambda}_j^0 < 0$, then Alg. \ref{alg:ERKLE-Proj} is used to project onto the PSD cone and $\hat{\lambda}_j^0 \gets 
\mathrm{max}\left(0,\lambda_{\downarrow}\right)$.  Otherwise, no projection is performed.  In the case where $\lambda_0^0 >> -3\gamma_j$, this simple lower-bounding procedure can save many 
eigenvalue/eigenvector computations until a projection may be necessary.
%
\begin{algorithm}[t]
  \center \caption{Efficient PSD Projection}
  \label{alg:ERKLE-Proj}
  \begin{algorithmic}[1]
     \Procedure{$\Pi_+^1$}{$\mathbf{K}$}
        \State Find $\lambda_{\downarrow}$ and $\mathbf{v}_{\downarrow}$ from $\mathbf{K}$
        \If{$\lambda_{\downarrow} < 0$}
            \State \textbf{return} $\mathbf{K} - \lambda_{\downarrow} \mathbf{v}_{\downarrow}\mathbf{v}_{\downarrow}^T$
        \Else
            \State \textbf{return} $\mathbf{K}$ 
        \EndIf
     \EndProcedure
  \end{algorithmic}
\end{algorithm}

\subsection{Passive-Aggressive Updates}
\label{subsec-pa}

A key difference between the batch and stochastic RCKL updates is the magnitude of the updates.  For both methods the magnitude of the updates with respect to a single triplet $t$ is a function of a 
learning rate and how well the previous solution satisfies $t$.  In the previous section we denoted the magnitude of an ERKLE update as $\gamma_j$.  In the batch setting, the same learning rate $\delta_i$ is used for all triplets in a given step.  In contrast, stochastic methods 
typically use different learning rates $\delta_j$ for different triplets $t_j$, which can result in faster convergence rates.
To take advantage of faster convergence, the learning rates must satisfy certain conditions.
Early work \cite{bottou1998online} on the topic of learning rates suggest that $\delta_j$ should satisfy two constraints: $\sum_{j=1}^{\infty}\delta_j^2 < \infty$ and $\sum_{j=1}^{\infty}\delta_j = 
\infty$.  For example $\delta_j = 1/j$ satisfies these constraints.  Later work \cite{moulines2011non} suggests a more aggressive setting of $\delta_j = 1/\sqrt{j}$.

However, in the online setting there is no reason to believe that a triplet should have less influence on the kernel than those obtained before it.  On the other hand, we do not wish to over-fit to the most recently obtained triplets.  It is this observation that motivates Passive-Aggressive (PA) Online Learning \cite{crammer2006online}.  In the RCKL setting, the general idea is that if the previous solution $\mathbf{K}_{j-1}$ satisfies a newly obtained triplet $t_{j} = \left(a,b,c\right)$ by a margin of 1, then do not update the kernel (passive).  Otherwise, update the kernel so that the kernel is changed the minimal amount, but $t_{j}$ is satisfied by a margin of 1 (aggressive).   A fortunate side effect of choosing minimally sized updates is that updates are less likely to result in non-PSD matrices than larger steps, thus potentially reducing the number of projections onto the PSD cone via our conservative eigenvalue estimate (Section 4.2).

To derive a passive-aggressive update for ERKLE, we wish to learn a magnitude of a stochastic step $\gamma_j = \delta_jf(\mathbf{K}_{j-1},t_j)$ with passive-aggressive properties. $f$ as defined by GNMDS in \eqref{eq:GNMDS} is inherently passive, but if $\mathbf{K}_{j-1}$ does not satisfy the margin constraint, it takes a step independent of how close the previous solution is to satisfying $t_j$.  As such, we wish to find a $\delta_j$ that takes an aggressive step.  We do this by solving the following optimization problem:
 \begin{equation}
   \begin{array}{rl}
     \displaystyle \min_{\delta_j} & \displaystyle \delta_j^2\\ [5 pt]
     \mathrm{s.t.} & d^2_{\mathbf{K}_j'}(a,b) + 1 \leq d^2_{\mathbf{K}_j'}(a,c), \delta_j \geq 0
   \end{array}
   \label{eq:paUpOpt}
 \end{equation}
\noindent By \eqref{eq:genUp} and \eqref{eq:GNMDS}, the first constraint can be rewritten as:
\begin{equation}
 d^2_{\mathbf{K}_{j-1}}(a,b) - d^2_{\mathbf{K}_{j-1}}(a,c) - 10\delta_j + 1 \leq 0
 \label{eq:PAConst}
\end{equation}
\noindent With the assumption that the triplet is not satisfied by a margin of one in $\mathbf{K}_{j-1}$, no update is required; otherwise, only a positive value of $\delta_j$ can 
satisfy \eqref{eq:PAConst}, making the positive constraint on $\delta_j$ redundant.  Also, the smallest $\delta_j$ that satisfies \eqref{eq:PAConst} is the one that makes the left hand side 
exactly zero.  As a result, the inequality constraint can be handled as equality.  To find the optimum we first write the Lagrangian $\mathcal{L}\left(\delta_j, \alpha\right)$:
\begin{equation}
 \delta_j^2 + \alpha\left(d^2_{\mathbf{K}_{j-1}}(a,b) - d^2_{\mathbf{K}_{j-1}}(a,c) - 10\delta_j + 1\right)
 \label{eq:PALag}
\end{equation}
\noindent Taking the partial derivative of \eqref{eq:PALag} with respect to $\delta_j$, setting it to 0, and solving for $\delta_j$ results in $\delta_j = 5\alpha$.  Substituting this back into \eqref{eq:PALag} makes the Lagrangian:
\begin{equation}
 -25\alpha^2 + \alpha\left(d^2_{\mathbf{K}_{j-1}}(a,b) - d^2_{\mathbf{K}_{j-1}}(a,c) +1\right)
 \label{eq:PALagAlpha}
\end{equation}
\noindent Taking the partial derivative of \eqref{eq:PALagAlpha} with respect to $\alpha$, setting it to 0, solving for $\alpha$ and then substituting this back into $\delta_j = 5\alpha$ results in 
the minimum step size that satisfies the margin constraint:
\begin{equation}
 \delta_j = \frac{d^2_{\mathbf{K}_{j-1}}(a,b) - d^2_{\mathbf{K}_{j-1}}(a,c) +1}{10}
 \label{eq:PAStep}
\end{equation}
%
%

A similar passive-aggressive update can be derived using the probability of a triplet being satisfied in STE.  Consider the following optimization:
 \begin{equation}
   \begin{array}{rl}
     \displaystyle \min_{\delta_j} & \displaystyle \delta_j^2\\ [5 pt]
     \mathrm{s.t.} & p_{t_j}^{\mathbf{K}_j'} \geq P, \delta_j \geq 0
   \end{array}
   \label{eq:paUpOptProb}
 \end{equation}
In \eqref{eq:paUpOptProb} the minimal step size is chosen such that the probability that a triplet is satisfied after the update is greater than or equal to a given probability $P \in 
\left(0.5,1\right)$.  Using \eqref{eq:paUpOptProb}, we derive the following step size:
\begin{equation}
 \delta_j = \frac{d^2_{\mathbf{K}_{j-1}}(a,b) - d^2_{\mathbf{K}_{j-1}}(a,c) + \kappa}{10}
 \label{eq:ProbStep}
\end{equation}
where $\kappa = \log\left(P\right) - \log\left(1-P\right)$. The full derivation is given in Sec. \ref{sec:STE-PA}.  Both derivations reveal that passive-aggressive updates using STE and GNMDS are very similar.  Setting $P = \frac{e}{1+e}$ in \eqref{eq:ProbStep} recovers the GNMDS passive-aggressive step in \eqref{eq:PAStep}, and changing the margin in \eqref{eq:PAStep} recovers different settings of $P$.  

Note that using \eqref{eq:PAStep} as a step size results in a $\mb{K}_j'$ with the intended passive-aggressive property, not necessarily the kernel $\mb{K}_j$ after the projection.
We choose to find a passive-aggressive step size instead of a full update for computational efficiency
Finding a true passive-aggressive step size with respect to $\mb{K}_j$ would require iteratively projecting onto the PSD cone, which is computationally prohibitive in the online setting.
In practice, $d^2_{\mb{K}_j'}$ is a good approximation to $d^2_{\mb{K}_j}$, as their difference is dependent on the magnitude of the (potentially) negative eigenvalue of $\mb{K}_j'$, which tends to be quite small.
  
Even for a proper setting of $\delta_j$, it has been shown that stochastic methods perform best when multiple rounds of updates or passes  are performed on the observed samples \cite{bordes2008sequence,recht2011hogwild,wang2012breaking}.  For our problem setting, this indicates that ERKLE may benefit from revisiting triplets that were previously used to update the kernel.  In our experiments we perform a simple multi-pass scheme where for each new triplet, ERKLE not only steps over the most recently obtained triplet, but also a number of randomly sampled triplets from the set of  previously obtained triplets.  We denote the number of ``passes'' ERKLE performs each time a new triplet is observed as $\beta$.  Algorithm \ref{alg:ERKLE-RG} in Sec. \ref{sec:MP} describes this process in more detail. This simple approach is sufficient for maintaining high accuracy while still ensuring computational efficiency for the online setting.

\begin{figure*}[t]
  \centering
  \begin{subfigure}{0.33\textwidth}
    \centering
    \includegraphics[width=\textwidth]{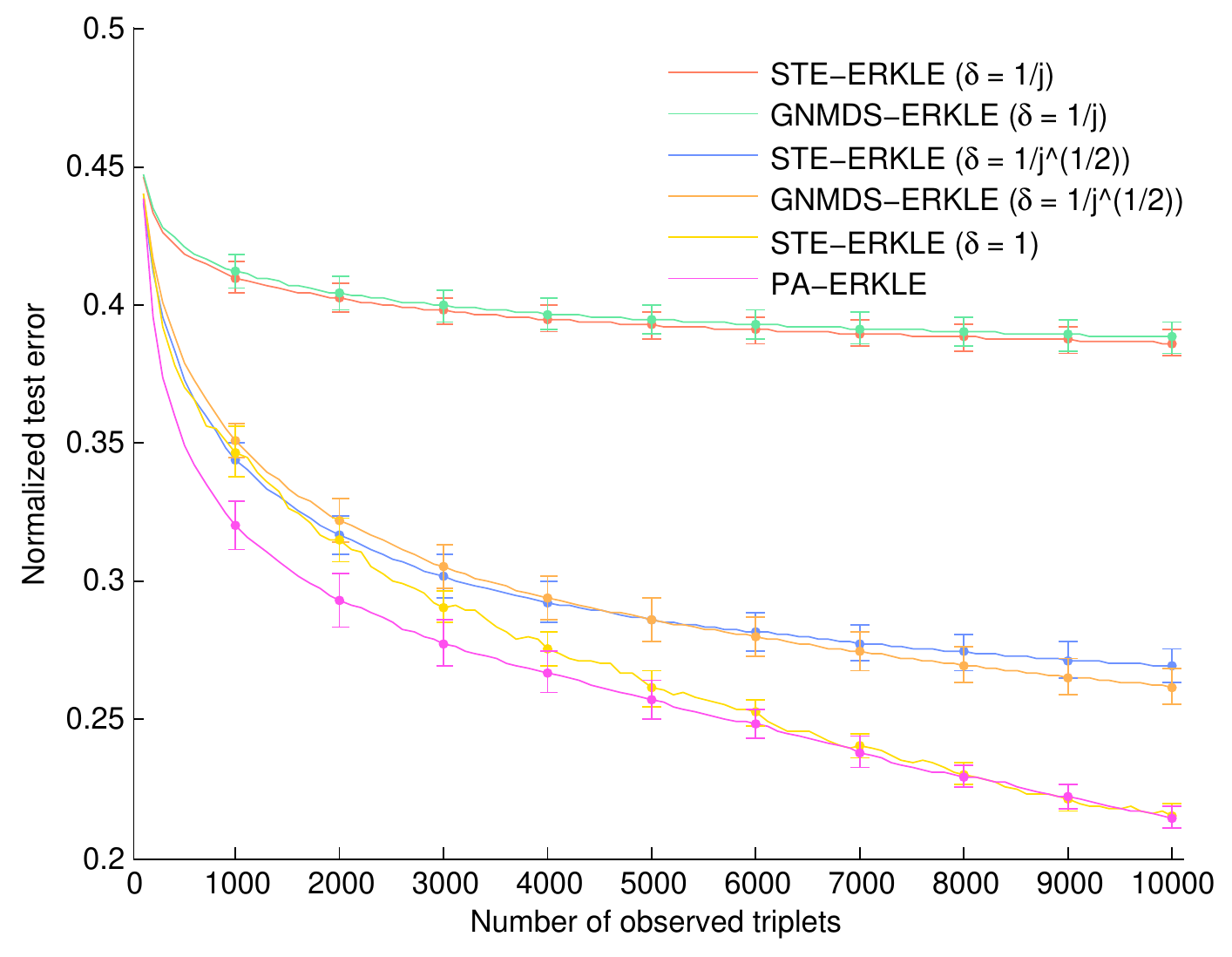}
    \caption{Comparison of ERKLE learning rates $\delta$}
    \label{fig:1a}
  \end{subfigure}%
  \hfill %
  \begin{subfigure}{0.33\textwidth}
    \centering
    \includegraphics[width=\textwidth]{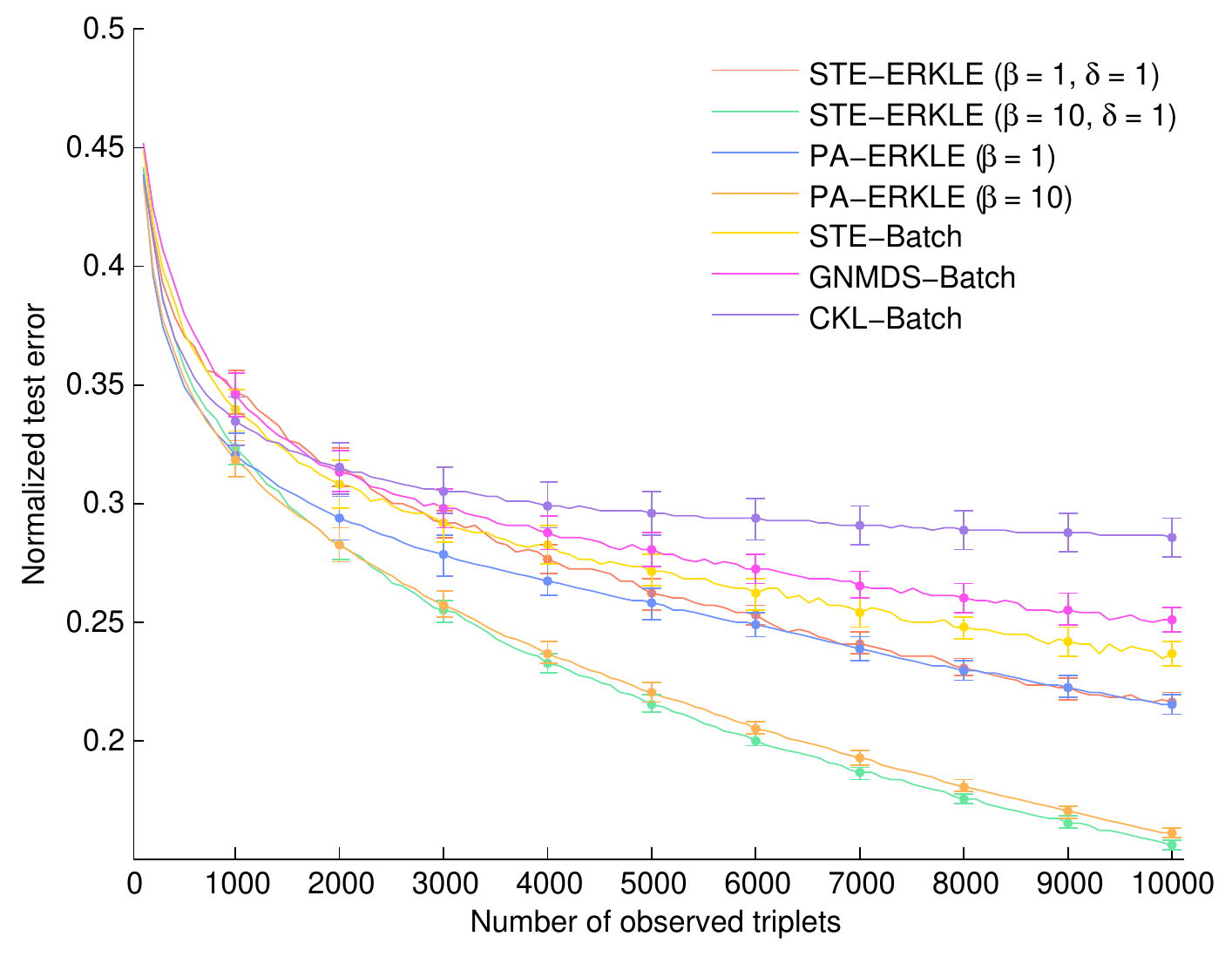}
    \caption{Test Error vs. \# of observed triplets}
    \label{fig:1b}
  \end{subfigure}%
  \hfill
  \begin{subfigure}{0.33\textwidth}
    \centering
    \includegraphics[width=\columnwidth]{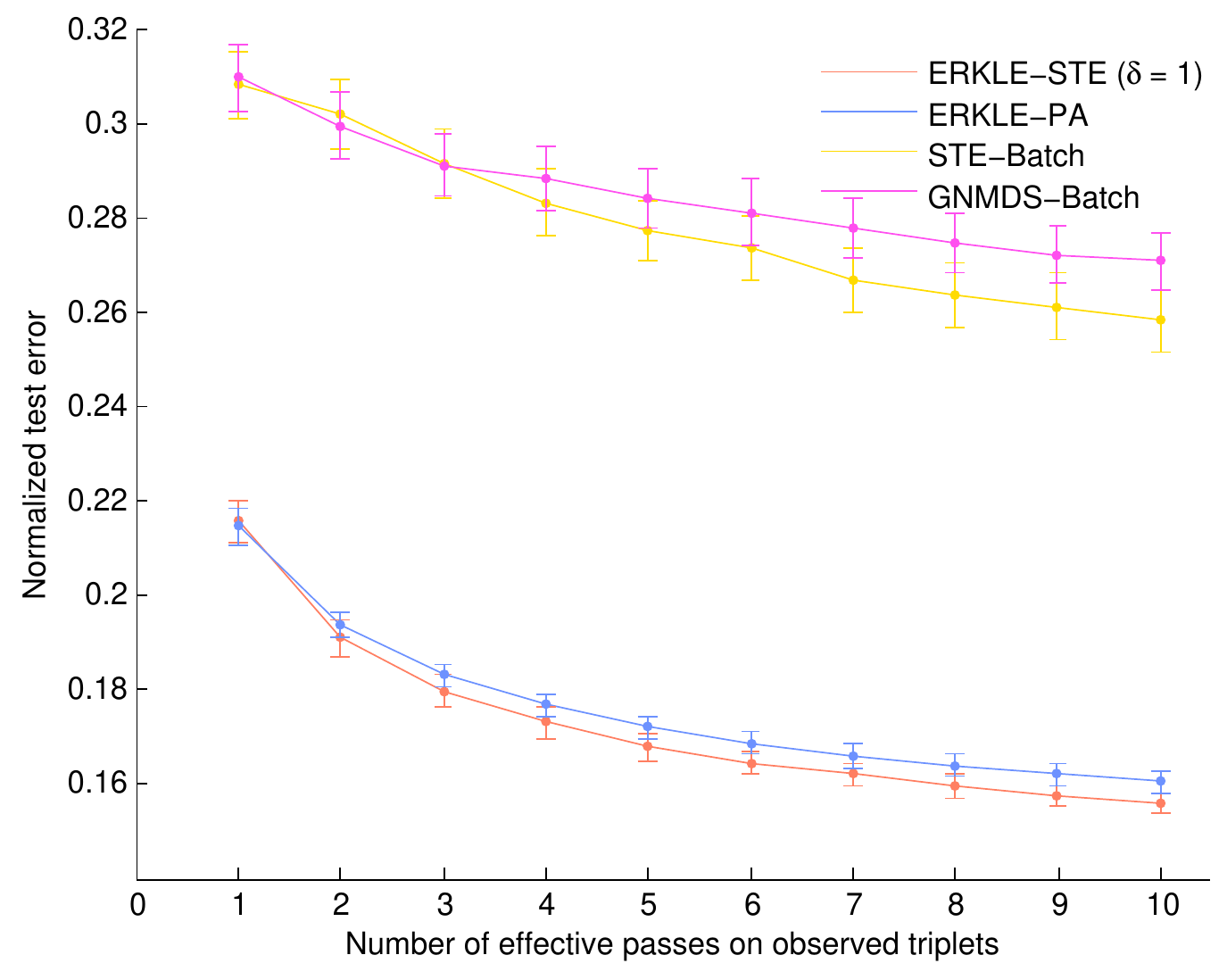}
    \caption{Test Error vs. \# of effective passes}
    \label{fig:1c}
  \end{subfigure}%
    \caption{Results from experiments on the small synthetic data set (10 trials)}
    \label{fig:1}
\end{figure*}

\section{Experiments}
\label{sec:Experiments}
In this section, we evaluate ERKLE by comparing it to batch RCKL methods.  Batch methods are not truly applicable to the online learning setting, but can be applied in what is often called ``mini-batches''.  In the mini-batch learning setting, every time a new batch of $m$ triplets are received, batch RCKL is run on all obtained triplets so far.  Thus, we compare ERKLE to running their batch counterparts in mini-batches.

We evaluate each method on four different data sets, each with its own challenges.  First, we start with a small-scale synthetic experiment to evaluate how the methods perform in an idealized setting.  Second, a large-scale synthetic experiment is run to show how ERKLE and batch compare in terms of practical run time.  Third, a data set of triplets over popular music artists is used to evaluate how the methods perform in a real-world setting with moderate triplet noise.  Finally, ERKLE and batch RCKL are evaluated on a data set of triplets over scene images, which  consist of a small number of triplets, thus focusing on the performance of these methods with very little feedback.  

For these experiments, we wish to see how the learned kernels generalize to held out triplets as triplets are obtained.  This is important in real-world applications where the goal is to accurately model all the relationships among objects, not just the observed ones.  Because of this, one of our main evaluation metrics is \emph{normalized test error}, which is defined as the total number of unsatisfied test triplets by a learned kernel divided by the total number of test triplets.

Unless otherwise noted, the experiments were run with the following specifications.  Each method started with an initial kernel set to identity in order to give no method an advantage (all methods initially satisfy no triplets).  All batch methods were terminated  after a maximum of 1000 iterations or when the change in objective between iterations was less than $10^{-7}$.  We denote the batch methods with the suffix ``-Batch'' (e.g. STE-Batch) and the ERKLE variants with ``-ERKLE'' (e.g. STE-ERKLE).  We denote passive-aggressive ERKLE as PA-ERKLE, and use the step size that satisfies the margin by 1 as in \eqref{eq:PAStep}.  The mini-batch size is 100, and all methods are evaluated every 100 observed triplets.  

We used the batch STE, GNMDS, and CKL (Crowd Kernel Learning \cite{tamuz2011adaptively}) MATLAB implementations specified by \cite{van2012stochastic} in which the \textit{eig} MATLAB function is used to find the all eigenvalues and eigenvectors for projection onto the PSD cone.  ERKLE  was also implemented in MATLAB, where the \textit{eigs} function is used to find a single eigenvalue/eigenvector pair with the smallest eigenvalue for the projections.  The $\tau$ hyperparameter was chosen to be the best performing setting over ten varying options.  The timed experiments were performed on an Intel Core i5-4670K CPU @ 3.4 GHz with 16 GB of RAM and the single thread option enabled.  Each experiment was performed with ten trials, each with different, randomly chosen test, train and validation sets.  The error bars in the graphs represent the 95\% confidence interval.
 

\subsection{Small-Scale Synthetic Data}
Our first experiment is to test each method on an ideal, small-scale, synthetic data set.  We created the synthetic data set by first generating 100 data points ($n = 100$) in $\mathbb{R}^{50}$ from $\mathcal{N}\left(0,1\right)$.  Using the distances between points, we answered all possible relative comparison queries which resulted in 485100 triplets.  10000 triplets were used as the train set and the rest were used as the test set.

%
\begin{figure*}[t]
  \centering
\begin{subfigure}{0.5\textwidth}
    \centering
    \includegraphics[width=\textwidth]{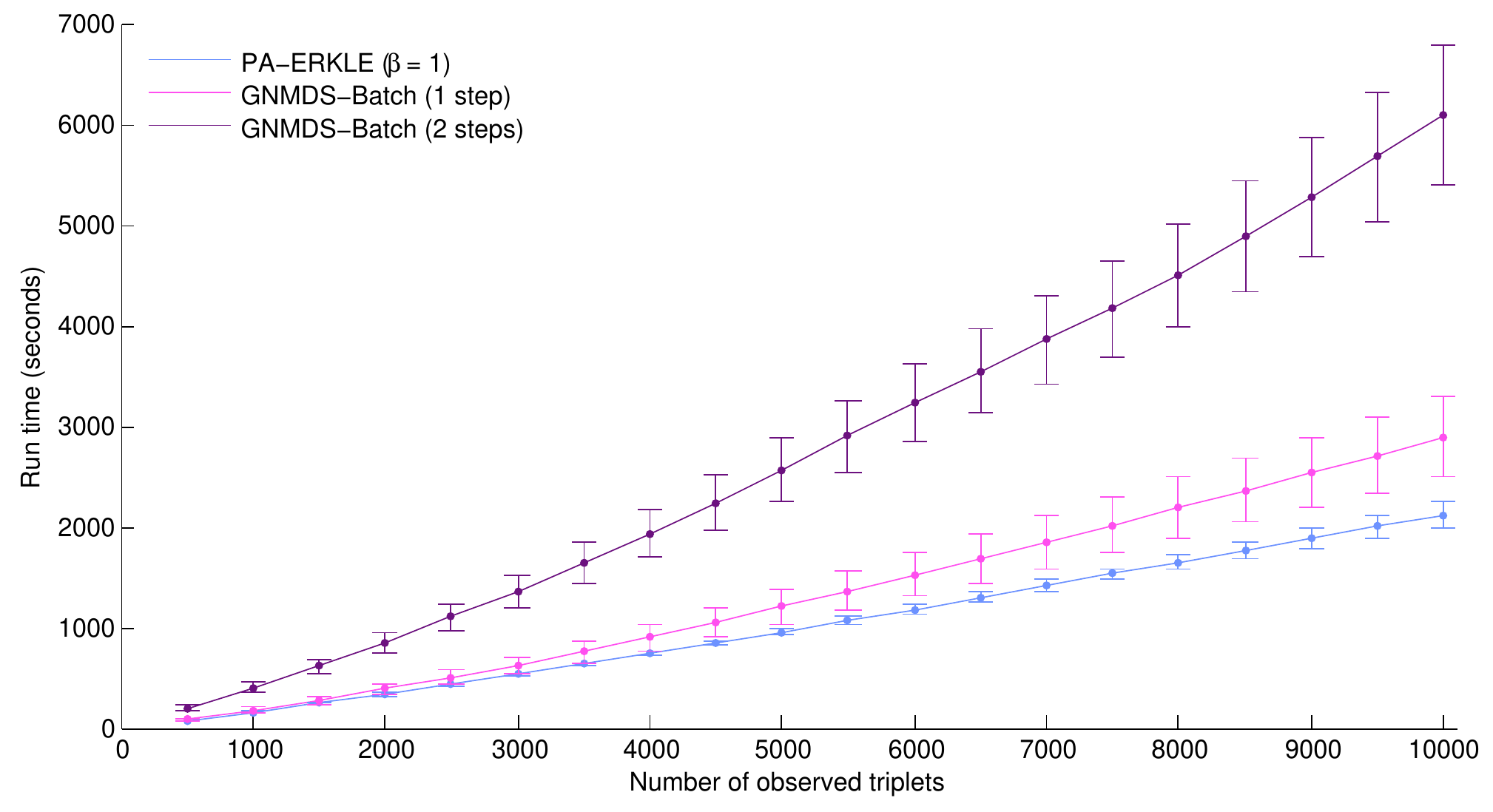}
    \caption{Run time (seconds) vs. number of observed triplets}
    \label{fig:2a}
  \end{subfigure}%
  \hfill %
  \begin{subfigure}{0.5\textwidth}
    \centering
    \includegraphics[width=\textwidth]{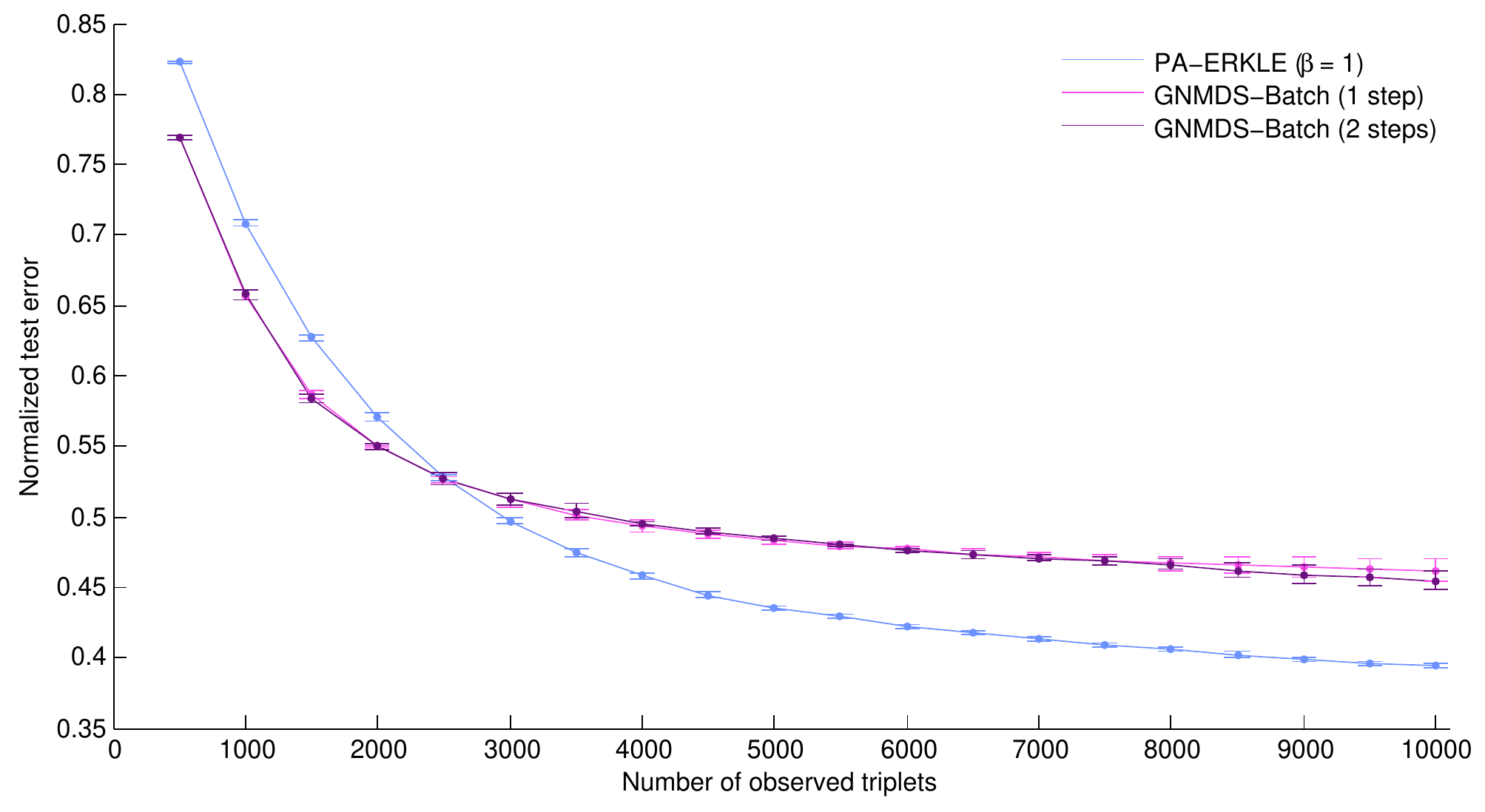}
    \caption{Test error vs. \# of observed triplets}
    \label{fig:2b}
  \end{subfigure}%
    \caption{Results from experiments on the large-scale synthetic data set (5 trials)}
    \label{fig:2}
\end{figure*}
\paragraph{Discussion:} Figure \ref{fig:1a} shows the effect that the learning rate parameter $\delta_j$ has on the performance of ERKLE as more triplets are observed in an online fashion.   For a setting of $1/j$, the learning rate decays too rapidly to improve performance significantly after $j = 3000$.  The learning rate $1/\sqrt{j}$ performs better, but still levels off, quicker than the final two methods.  The last two methods have learning rates that are independent of the number of observed triplets.  STE-ERKLE with a constant learning rate and PA-ERKLE take steps solely based on how well the current solution satisfies the observed triplet, and vastly outperform the alternative learning rates based on number of observations.  This result indicates that reducing the influence of a triplet because it was observed later has an adverse effect on the ability of a learned kernel to generalize to unobserved triplets.

Figure \ref{fig:1b} shows the performance of STE-ERKLE (with $\delta_j$ set to 1), and PA-ERKLE compared to three batch RCKL methods.  The $\tau$ hyperparameter was chosen by selecting the best setting over choices as evaluated on the test set.  With a single pass over the data ($\beta = 1$), both ERKLE methods outperformed all batch methods slightly.  With ten passes over the data, the ERKLE methods outperformed the batch methods by a large margin. In addition, the batch methods level off more quickly than the ERKLE methods, indicating that if more triplets were obtained, the ERKLE methods would further outperform even the batch methods.  We believe that these results show that by minimizing the expected risk directly, ERKLE is able to learn a more general kernel than batch methods that minimize empirical risk.  

Figure \ref{fig:1c} shows the performance of two ERKLE methods and two batch RCKL methods as a function of how many effective ``passes'' each method performed on the data.  For ERKLE, this amounts to the setting of the $\beta$ parameter.  For the batch RCKL methods, this is the number of full gradient steps it takes.  Each method was run over all training triplets with the step size $\delta_j$ validated on the test set for the batch methods.
This effectively measures training cost as a function of passes through the data, and thus, is independent of implementation.  Clearly, if only few passes through the data can be performed, then ERKLE is the better choice.

\subsection{Large-Scale Synthetic Data}
Next, we evaluated how PA-ERKLE compared to batch GNMDS in terms of practical run time on a large scale experiment. For this experiment, we generated 5000 data points in the same manner in which the small-scale synthetic data was generated. 10000 randomly generated triplets were used as the train set and 50000 were used as the test set.  The batch methods were run in mini-batches of 500 triplets due to time constraints.  The hyperparameter $\tau$ and the step size $\delta_i$ were chosen as the settings that best performed on the test set.  This experiment was run over 5 trials, each with a different train and test set.

\paragraph{Discussion:} Figure \ref{fig:2a} shows the cumulative run time of one pass of PA-ERKLE, and 1 and 2 steps of batch GNMDS.
The times shown for the batch methods are for the best chosen $\tau$ and not for the total time it took to find it.
The figure shows that a single pass of PA-ERKLE is often significantly faster than a single gradient step of batch GNMDS.
Two steps of GNMDS takes even longer.
ERKLE can perform online updates much faster due to the efficient projection procedure as well as the ability to skip certain projections by estimating the lower bound.
In this experiment, the mean number of eigenvalue/eigenvector computations over the 5 trials was 724.2 with a standard deviation of 3.7.
Hence PA-ERKLE was able to skip the projection step roughly 93\% of the time.

Figure \ref{fig:2b} depicts the test errors of each method.
Initially, the batch methods perform better, but around 2500 triplets, PA-ERKLE outperforms the batch methods.
This experiment indicates that PA-ERKLE can achieve competitive results with batch methods in a single pass over the data, and produce truly online solutions instead of mini-batch solutions while having faster run time.

\subsection{Music Artist Similarity}
For the last two experiments we performed evaluations on real-world data sets.  First, we performed an experiment using relative comparisons among popular music artists gathered from a web survey.  The \emph{aset400} data set \cite{ellis2002quest} contains 16,385 relative comparisons over 412 artists.  We randomly chose 10000 triplets as the train set, 1000 as the validation set for the $\tau$ parameter, and the rest were used as the test set.  The aset400 data set presents a challenge not present in the synthetic data:  It has a moderate amount of conflicting triplets, thus methods used in the evaluation must deal with noise within the triplets.

\begin{figure*}[t]
  \centering
  \begin{subfigure}{0.33\textwidth}
    \centering
    \includegraphics[width=\textwidth]{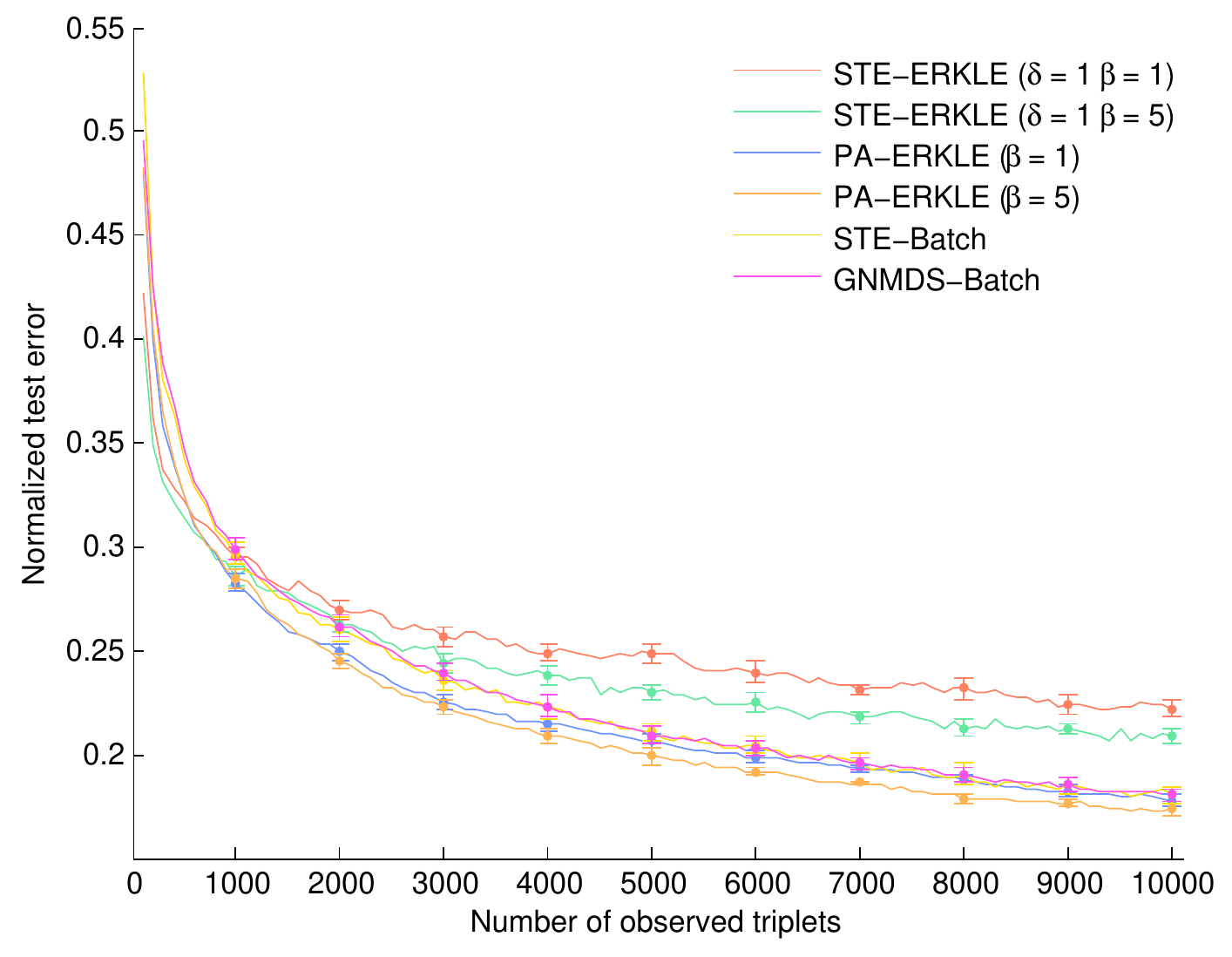}
    \caption{Test error vs. number of observed triplets}
    \label{fig:3a}
  \end{subfigure}%
  \hfill %
  \begin{subfigure}{0.33\textwidth}
    \centering
    \includegraphics[width=\textwidth]{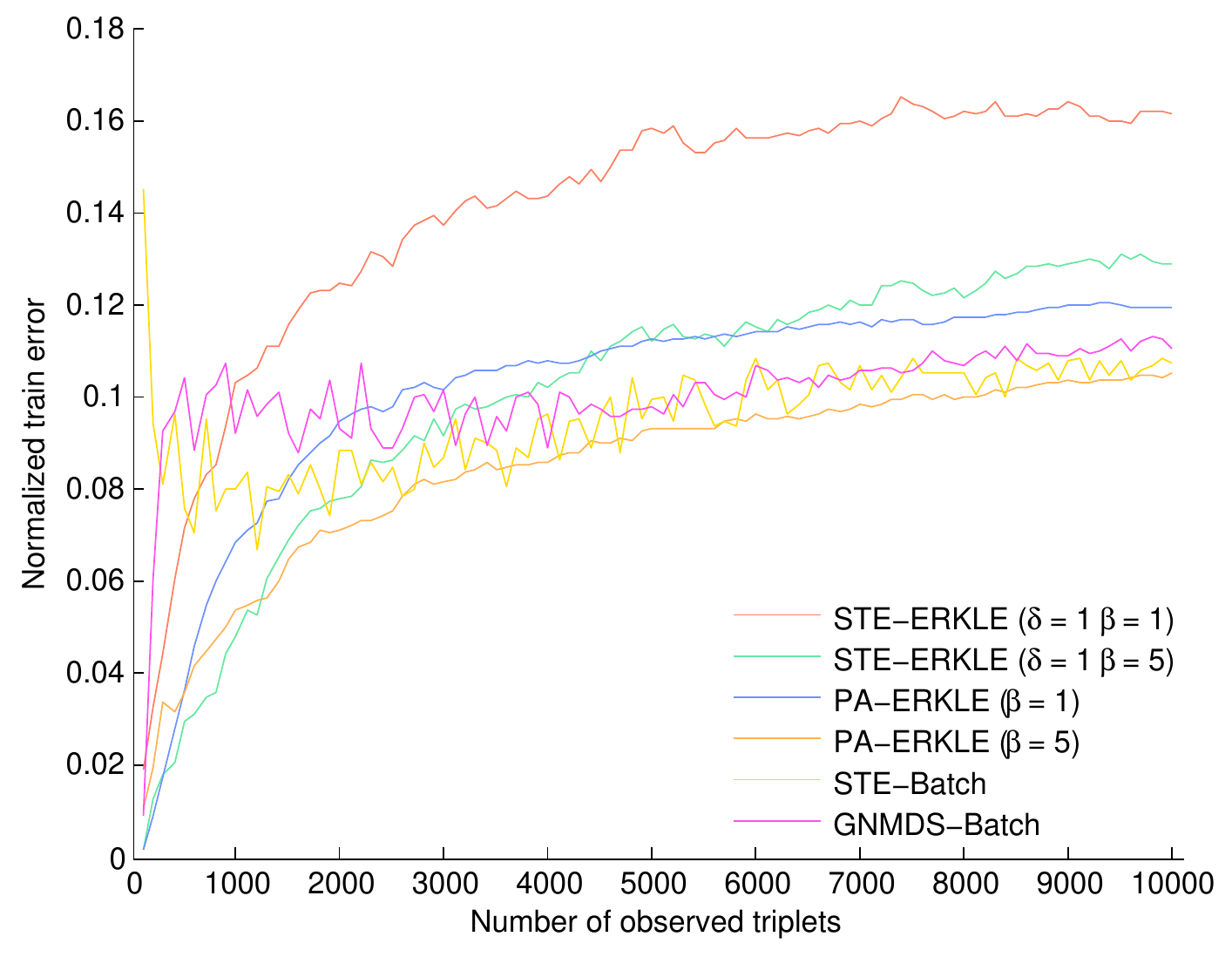}
    \caption{Train error vs. \# of observed triplets}
    \label{fig:3b}
  \end{subfigure}%
  \hfill
  \begin{subfigure}{0.33\textwidth}
    \centering
    \includegraphics[width=\columnwidth]{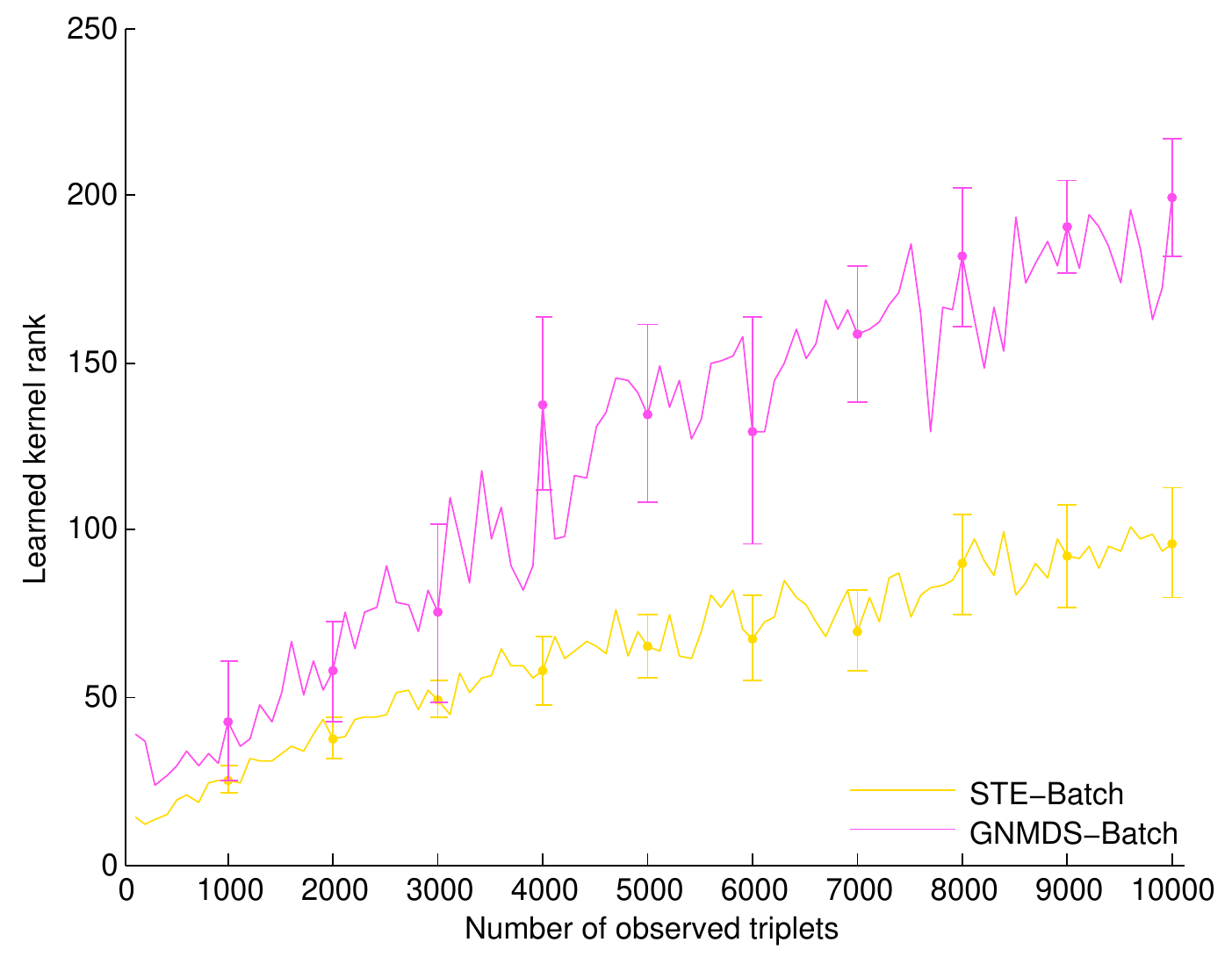}
    \caption{Kernel rank vs. \# of observed triplets}
    \label{fig:3c}
  \end{subfigure}%
    \caption{Results from experiments on the aset400 data set (10 trials)}
    \label{fig:3}
\end{figure*}

\paragraph{Discussion:} Figure \ref{fig:3a} shows how ERKLE and batch RCKL methods generalize to the test set.  STE-ERKLE performs considerably worse than the other methods, most likely due to the noise in the observed triplets.  The probability $p_t^{\mathbf{K}}$ used in STE-ERKLE decays rapidly.  Thus, triplets that are in agreement with previously obtained triplets do not influence the learned kernel greatly.  However, a conflicting triplet will make STE-ERKLE perform a relatively more drastic update.  PA-ERKLE, however, is much more robust to noise due to the minimal step size taken to satisfy a triplet.  Because of this, PA-ERKLE performs as well as the batch methods and often better when multiple passes are taken.  


Figure \ref{fig:3b} shows the training errors of each method.  This figure highlights how well each method fits to the observed triplets.  The STE-ERKLE models are greatly effected by the presence of conflicts in that they do not learn a kernel that fits to a large number of the observed triplets.  PA-ERKLE, on the other hand, is able to fit better to the set of observed triplets, thus resulting in better test accuracy, as well.

As previously discussed, dissimilar from batch methods ERKLE does not use trace regularization.
Experimentally, however, we nevertheless find that our method outperforms batch methods that use trace regularization, in either producing low-rank or high-rank kernels.
To demonstrate this, in Figure \ref{fig:3c} we plot the ranks of the kernels learned by the batch methods.
In our experiments, the range of potential $\tau$ values was set so that the batch methods never chose either the upper or lower bound.
We did this to ensure that the range of regularization options were sufficiently strict or lenient.
We observe that the batch methods generally produce low-rank kernels under a small number of triplets, but as the number of triplets are observed the rank increases.
Our method is able to better generalize without using trace regularization, regardless of the preferred rank, due to the PA updates only satisfying triplets to the necessary extent.


\subsection{Outdoor Scene Similarity}
Our final experiment used triplets over 200 randomly chosen images of scenes from the Outdoor Scene Recognition (OSR) data set \cite{oliva2001modeling}.  Relative comparison queries were posed to 20 people via an online system.  After an initial 1200 randomly chosen queries were answered (every object appeared as the head of a triplet 6 times), 20 ``rounds'' of 200 triplets were adaptively chosen according to the adaptive selection criterion in \cite{tamuz2011adaptively}, resulting in 3600 total triplets.  For each trial of this experiment, 1000 triplets were randomly chosen as the test set, 1000 as the train set, and 600 were used as the validation set for the $\tau$ parameter.  This experiment is especially challenging for two reasons.  First, this is the smallest experiment in terms of triplets, highlighting how the methods perform with little feedback.  In addition, the adaptive selection algorithm chooses relative comparison queries with the highest information gain, meaning, the triplets are intentionally chosen to give disparate information about the relationships among objects.

\paragraph{Discussion:}
Figure~\ref{fig:4a} depicts test errors on each method.
We observe that STE-ERKLE consistently outperforms STE-Batch, and in particular STE-ERKLE performs well under a small number of
triplets relative to all other methods.
PA-ERKLE is comparable or outperforms its batch counterpart in GNMDS-Batch, given enough triplets (at least 500).
However,  PA-ERKLE performs quite well in training error compared to all other methods, indicating that even in such a challenging
scenario, the passive-aggressive update scheme minimally interferes with previously obtained triplets.

\begin{figure*}[t]
  \centering
  \begin{subfigure}{0.5\textwidth}
    \centering
    \includegraphics[width=\textwidth]{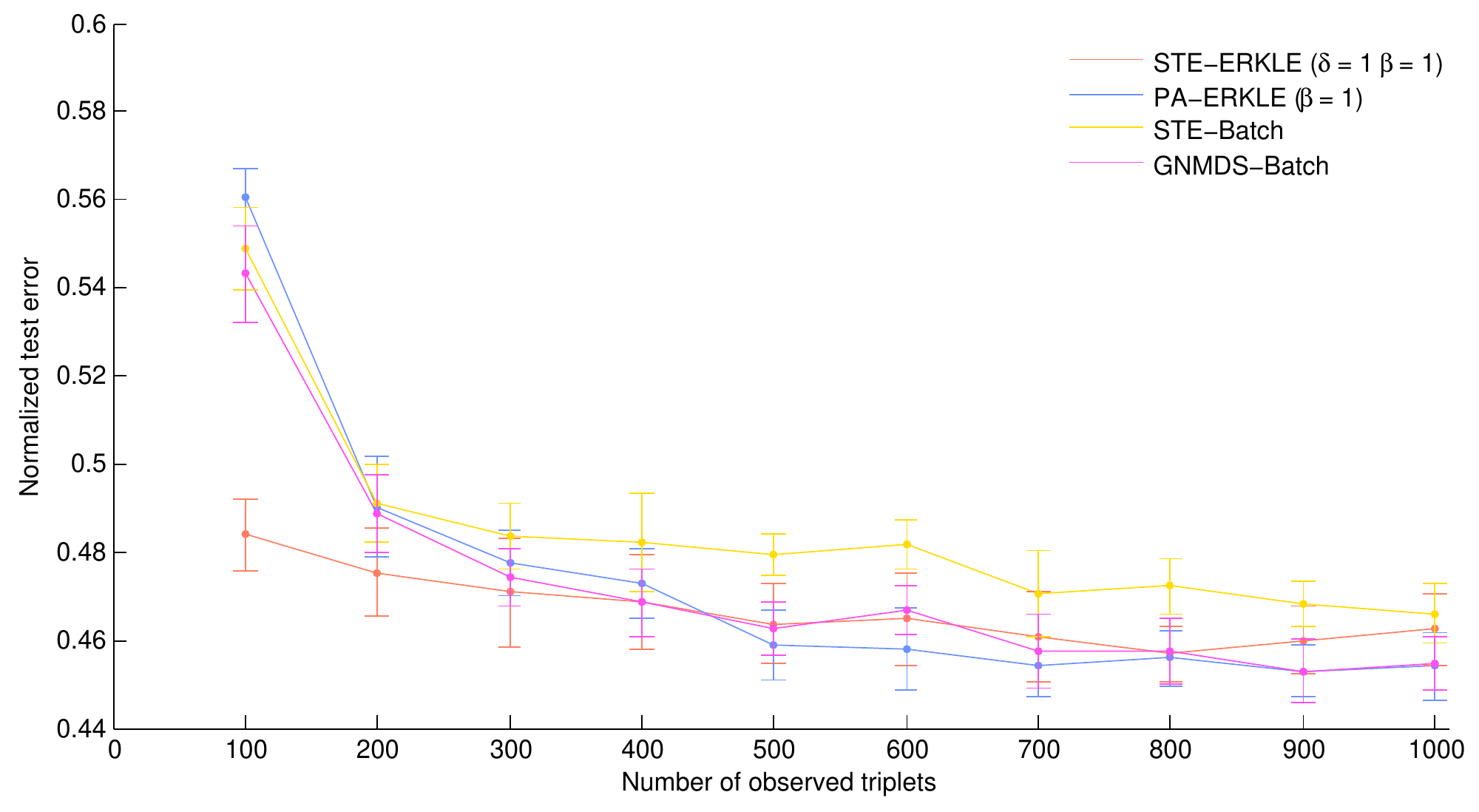}
    \caption{Test error vs. \# of observed triplets}
    \label{fig:4a}
  \end{subfigure}%
  \hfill %
  \begin{subfigure}{0.5\textwidth}
    \centering
    \includegraphics[width=\textwidth]{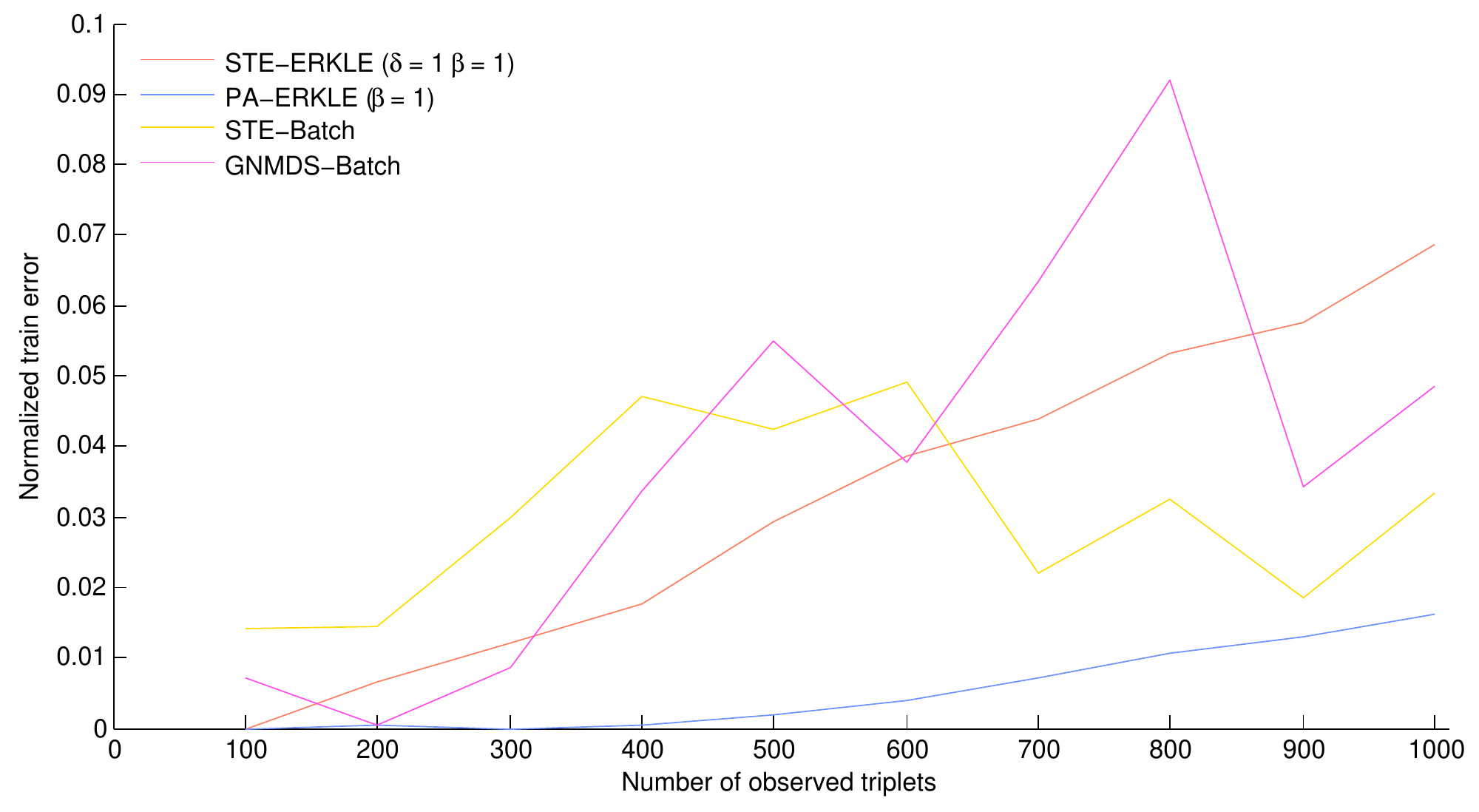}
    \caption{Train error vs. \# observed triplets}
    \label{fig:4b}
  \end{subfigure}%
    \caption{Results from experiments on the OSR data set (10 trials)}
    \label{fig:4}
\end{figure*}

\section{Conclusion and Future Work}
\label{sec:Con}
In this work, we developed a method to learn a PSD kernel matrix from relative comparisons given in an online fashion.
By taking advantage of the sparse and low-rank structure of the online formulation, we show how to take stochastic gradient descent updates of complexity $O(n^2)$.
We show how passive-aggressive online learning benefits our method in terms of generalizing to unseen triplets, and in conjunction with the stochastic gradient structure, enables us to perform a small number of necessary PSD projections in practice.
Experimentally, we show on synthetic and real-world data that our method learns kernels that generalize as well and often better to held out relative comparisons than batch methods, while demonstrating improved run-time performance.

For future work, we wish to improve online RCKL in three ways.
First, will explore the use of online trace regularization.
If trace regularization is naively applied to the stochastic gradient in \eqref{eq:stoUp}, the update becomes full-rank and our efficient projection procedure cannot be used.
However, an efficient update scheme should be possible if the kernel itself is low-rank.
We will investigate novel methods for appropriately weighting the trace in an online manner, so that we are consistent with the parameter-free property of PA-ERKLE.
Second, PA-ERKLE performed well in our experiments with moderate triplet noise, however, it could be beneficial to explicitly handle conflicting triplets when they are observed.
This can be done out of model using a denoising method~\cite{mcfee2011learning}, or in model using a threshold on the passive-aggressive learning rate.
Finally, one of the main benefits of having an online learning algorithm is the natural application of active learning methods.
Prior work has proposed an adaptive selection scheme which operates in mini-batches~\cite{tamuz2011adaptively}; however, such a scheme is too expensive to be applied online.
We will investigate novel adaptive triplet selection methods which are both efficient and informative.

\appendix
\section{Derivation of STE Passive-Aggressive Step Size}
\label{sec:STE-PA}
To derive the STE version of the passive-aggressive step size we wish to solve the following optimization (4.19):
 \begin{equation*}
   \begin{array}{rl}
     \displaystyle \min_{\delta_j} & \displaystyle \delta_j^2\\ [5 pt]
     \mathrm{s.t.} & p_{t_j}^{\mathbf{K}_j'} \geq P, \delta_j \geq 0
   \end{array}
 \end{equation*}
\noindent As with the GNMDS derivation with the assumption that the triplet is not satisfied by a probability greater than or equal to $P$, only a positive value of $\delta_j$ can satisfy the first constraint, making the positive constraint on $\delta_j$ redundant.  In addition, the smallest $\delta_j$ that satisfies the remaining constraint is the one that makes the left hand side exactly zero.  As a result, the inequality constraint can be handled as equality.  Next, we take the Lagrangian:
%
\begin{equation*}
\begin{array}{rl}
\displaystyle \delta_j^2 & + \alpha\left(\log\left(P\right) - \log\left(1-P\right)\right)\\
 & + \alpha\left(d_{\mathbf{K}_{j-1}}(a,b) - d_{\mathbf{K}_{j-1}}(a,c)\right) -10\delta_j\alpha
\end{array}
\end{equation*}
\noindent Taking the partial derivative of the Lagrangian with respect to $\delta_j$, setting it to 0, and solving for $\delta_j$ results in $\delta_j = 5\alpha$.  Substituting this back into the Lagrangian makes it:
%
\begin{equation*}
\begin{array}{rl}
\displaystyle -25\alpha^2 & + \alpha\left(\log\left(P\right) - \log\left(1-P\right)\right) \\
& + \alpha\left(d_{\mathbf{K}_{j-1}}(a,b) - d_{\mathbf{K}_{j-1}}(a,c)\right)
\end{array}
\end{equation*}
\noindent Taking the partial derivative of the Lagrangian with respect to $\alpha$, setting it to 0, and solving for $\alpha$ results in:
$$
\alpha =  \frac{\log\left(P\right) - \log\left(1-P\right) + d_{\mathbf{K}_{j-1}}(a,b) - d_{\mathbf{K}_{j-1}}(a,c)}{50}
$$
Substituting this into the solution for $\delta_j$ gives us:
$$
\delta_j =  \frac{\log\left(P\right) - \log\left(1-P\right) + d_{\mathbf{K}_{j-1}}(a,b) - d_{\mathbf{K}_{j-1}}(a,c)}{10}
$$
This is what is given in (4.20).

\section{ERKLE with Multiple Passes}
\label{sec:MP}
\begin{algorithm}[h]
  \center \caption{ERKLE with Multiple Passes}
    \label{alg:ERKLE-RG}
    \begin{algorithmic}[1]
    \Require $\beta$ : \# of triplets stepped over
    \State $\mathbf{K}_0 \gets \mathbf{I}$
    \For{$j=1,2,...$}
       \State $\mathbf{K}_j' \gets \mathbf{K}_{j-1} - \delta_j\nabla l\left(\mathbf{K}_{j-1},t_j\right)$
       \State $\mathbf{K}_j \gets \Pi_{S_+}^1\left(\mathbf{K}_j'\right)$
       \If{$j > 2\beta$}
          \For{$k=1,2,...,\beta-1$}         
             \State Randomly select $t'$ from $\left\{t_1,t_2,...,t_j\right\}$
             \State $\mathbf{K}_j' \gets \mathbf{K}_{j} - \delta_{j+k}\nabla l\left(\mathbf{K}_{j-1},t'\right)$
             \State $\mathbf{K}_j \gets \Pi_{S_+}^1\left(\mathbf{K}_j'\right)$
          \EndFor
       \EndIf
    \EndFor
   \end{algorithmic}
 \end{algorithm}

\noindent Algorithm 1 is much like the original ERKLE algorithm.  Here, after a sufficient number of triplets have been obtained (in our experiments, we chose $2\beta$), $\beta - 1$ triplets are selected every iteration from all previously observed triplets (for a total of $\beta$ updates per iteration). These triplets are stepped over as done in the original ERKLE algorithm.  For our random selection used in our experiments, we simply selected uniformly at random with replacement from the obtained triplets.  More sophisticated random selection procedures may be used in order ensure triplets obtained initially do not get selected drastically more times than those obtained later.  For instance, when a triplet gets chosen on line 7, one could reduce the probability of that triplet being chosen subsequently.

\bibliography{ERKLE_arxiv}
\bibliographystyle{plain}
\end{document}